\documentclass{llncs}
\usepackage{llncsdoc}
\usepackage{graphicx}
\graphicspath{{images/} }
\makeatletter
\def\temp{dvips.def}
\ifx\Gin@driver\temp
\def\Ginclude@graphics#1{\def\temp{#1}---image \expandafter\strip@prefix\meaning\temp---}
\fi
\makeatother

\usepackage{epsf}
\usepackage{epsfig}

\usepackage{amsmath}
\usepackage{amssymb}
\usepackage{color}

\newcommand{\ignore}[1]{}

\newcommand{\E}{{\bf E}}

\newcommand{\RID}{{ {\rm \scriptscriptstyle RID}}}

\ignore{
\newtheorem{theorem}{Theorem}

\newtheorem{lemma}[theorem]{Lemma}

}

\begin{document}

\title{Non-Adaptive Randomized Algorithm\\ for Group Testing}
\author{{\bf Nader H. Bshouty}\\ Dept. of Computer Science\\ Technion,  Haifa, 32000\\
{\bf Nuha Diab} \\ Sisters of Nazareth High School \\ Grade 12\\ P.O.B. 9422, Haifa, 35661\\
{\bf Shada R. Kawar}  \\ Nazareth Baptist High School\\ Grade 11\\ P.O.B. 20, Nazareth, 16000\\
{\bf Robert J. Shahla} \\ Sisters of Nazareth High School \\
Grade 11\\ P.O.B. 9422, Haifa, 35661}
\institute{}
\maketitle

\begin{abstract}
We study the problem of group testing with a non-adaptive randomized algorithm in the random incidence design (RID) model where each entry in the test is chosen randomly independently from $\{0,1\}$ with a fixed probability $p$.

The property that is sufficient and necessary for a unique decoding is the separability of the tests, but unfortunately no linear time algorithm is known for such tests. In order to achieve linear-time decodable tests, the algorithms in the literature use the disjunction property that gives almost optimal number of tests.

We define a new property for the tests which we call semi-disjunction property. We show that there is a linear time decoding for such test and for $d\to \infty$ the number of tests converges to the number of tests with the separability property and is therefore optimal (in the RID model). Our analysis shows that, in the RID model, the number of tests in our algorithm is better than the one with the disjunction property even for small $d$.
\end{abstract}

\section{Introduction}\label{Int}
As quoted in~\cite{GT}, ``Robert Dorfman's paper in 1943,~\cite{D43}, introduced the field of Group Testing.
The motivation arose during the Second World War
when the United States Public Health Service and the Selective service embarked upon a large scale project.
The objective was to weed out all syphilitic men called up for induction.
However, syphilis testing back then was expensive and testing every soldier
individually would have been very cost heavy and inefficient. A basic breakdown of a test is:
Draw sample from a given individual,
perform required tests and
determine the presence or absence of syphilis.
Suppose we have $n$ soldiers. Then this method of testing leads to $n$ tests.
Our goal is to achieve effective testing in a
scenario where it does not make sense to test $100,000$ people to get (say) $10$ positives.
The feasibility of a more effective testing scheme hinges on the following property.
We can combine blood samples and test a combined sample together to
check if at least one soldier has syphilis''.

Group testing was originally introduced as a potential approach to the above economical
mass blood testing \cite{D43}. However it has been proven to be applicable in a variety of problems, including quality control in product testing \cite{SG59}, searching files in storage systems \cite{KS64}, sequential screening of experimental variables \cite{L62}, efficient contention resolution algorithms for multiple-access communication \cite{KS64,W85}, data compression \cite{HL02}, and computation in the data stream model \cite{CM03}. See a brief history and other applications in~\cite{Ci13,DH00,DH06,H72,MP04,ND00}
and references therein.

We now formalize the problem.
Let $S$ be the set of the $n$ {\it items} (soldiers, clones) and let $I\subseteq S$ be the set of {\it defective items} (the sick soldiers, positive clone).
Suppose we know that the number of defective items, $|I|$, is bounded by some integer~$d$.
A {\it test} or a {\it query} (pool or probing) is a set $J \subset S$. The answer to the test is $T(I,J)=1$ if $I\cap J\not=\O$ and $0$ otherwise. Here~$J$ is the set of soldiers for which their blood samples is combined. The problem is to find the defective items (sick soldiers)
with a minimum number of tests. We will assume that $S=[n]:=\{1,2,\ldots,n\}$ and will identify the test $J\subseteq S$ with an {\it assignment} $a^J\in\{0,1\}^n$ where $a^J_i=1$ if and only if $i\in J$, i.e., $a^J_i=1$ signifies that item $i$ is in test $J$. Then the answer to the test $a^J$ is $T(I,a^J)=1$ ({\it positive}) if there is $i\in I$ such that $a^J_i=1$ and $0$ ({\it negative}) otherwise.

In the {\it adaptive algorithm}, the tests can depend on the answers to the previous ones. In the {\it non-adaptive algorithm} they are independent of the previous one and; therefore, one can do all the tests in one parallel step. In all the above applications, non-adaptive algorithms are most desirable since minimizing time is of utmost importance.
The set of tests in any non-adaptive deterministic algorithm (resp. randomized) can be identified with a (resp. random) $m\times n$ {\it test matrix} $M$ (pool design) such that its rows are all the assignments $a^J$ of all the tests $J$ in the algorithm.
The above problem is also equivalent to the problem of non-adaptive deterministic (resp. randomized) learning the class of boolean disjunction of up to $d$ variables from membership queries~\cite{B13z}.

It is known that any non-adaptive {\it deterministic} algorithm must do at least $\Omega(d^2\log n/\log d)$ tests~\cite{DR82,F96,PR11,R94}.
This is $O(d/\log d)$ times more than the number of tests of the known non-adaptive randomized algorithms that do only $O(d\log n)$ tests. In this paper we study the case where the number of items $n$ is very large and the cost of the test is ``very expensive'' and in that case, deterministic algorithms are impractical.
Also, algorithms that run in super-linear time are not practical for group testing. So we need our algorithms to do a minimum number of tests and run in quasi-linear time, preferably, linear time in the {\it size of all the tests}, that is, $O(dn\log n)$.

Non-adaptive randomized algorithms have been proposed in~\cite{BBKT95,BKB95,DH06,ER63,H00,HL01}. The first folklore non-adaptive randomized algorithm known from the literature
simply chooses a random $m\times n$ test matrix $M$, for sufficiently large $m$, where each entry of $M$ is chosen randomly independently to be $0$ with probability $p$ and $1$ with probability $1-p$.
This type of algorithm is called a {\it random incidence design} (RID) and is the simplest type of randomized algorithm. In this paper we will study this type of algorithms.

If $I\subseteq S=[n]$ is the set of defective items then the vector of answers to the tests (rows of $M$) is $T(I,M):=\vee_{i\in I} M^{(i)}$ where $\vee$ is bitwise ``or'' and $M^{(i)}$ is the $i$th column of $M$.
The above parameters, $m$ and $p$, are chosen so that the probability of success of recovering the defective items (decoding) is maximal. If the number of defective items is bounded by $|I|\le d$ then in order to find the defective items with probability at least $1-\delta$, we need that with probability at least $1-\delta$, $M$ satisfies the following property: For every $J\subseteq [n]$, $|J|\le d$ and $J\not=I$ we have $T(J,M)\not=T(I,M)$. This basically says that the only set $J$ of up to $d$ items that is consistent with the answers of the tests $\wedge_{i\in I} M^{(i)}$ is $I$.
We call such a property $(I,d)$-{\it separable property} and $M$ is called an $(I,d)$-{\it separable test matrix}. Unfortunately, although this property is what we need to solve the problem, there is no known linear time (or even polynomial time) algorithm that finds the defective items from the answers of an $(I,d)$-separable test matrix. The only known algorithm is the trivial one that exhaustively goes over all possible up to $d$ sets $J$ and verifies whether $T(J,M)$ is equal to the answers of the tests. This is one of the reasons why the {\it separability property} has not been much studied in the literature.

Seb\"o, \cite{S85}, studies the easy case when the number of defective items is {\it exactly}~$d$. He shows that the best probability for the random test matrix (i.e., that gives a minimum number of tests) is $p=1-1/2^{1/d}$. In this paper, we study the general case when the number of defective items is at most $d$. We give a new analysis for the general case and show that the probability that gives minimum number of tests is $p=1-1/d$. In particular, we show that \begin{eqnarray}\label{fifi} \lim_{n\to\infty} \frac{m_{\RID}}{\ln n}\le \gamma_{d-1}+O\left(\frac{1}{d}\right) :=ed-\frac{e+1}{2}+O\left(\frac{1}{d}\right)\end{eqnarray} where $m_{\RID}$ is the minimum number of tests in the RID algorithms, $e=2.71828\cdots$ is the Euler's number and $\gamma_d$ is defined in (\ref{gamma}). But again, no polynomial time algorithm is known for such matrix even for the special case when the number of the defective items is exactly $d$.

To be able to find the defective items in {\it linear time} in the size of all the tests, that is, $O(dn\log n)$, and asymptotically optimal $O(d\log n)$ tests, algorithms in the literature use a relaxed property for~$M$. The folklore non-adaptive randomized algorithm randomly chooses $M$ such that with probability at least $1-\delta$ the following relaxed property holds: for each non defective item $i\in S\backslash I$ there is a test that contains it but does not contain the defective items. Such a test is surely negative (gives answer $0$) and is a witness for the fact that item $i$ is not defective. We call such a property an $(I,d,i)$-{\it disjunct} property where $I$ is the set of defective elements.
If the test matrix is $(I,d,i)$-disjunct for every non-defective item $i$ then we call the test matrix $(I,d)$-{\it disjunct}.
When the test matrix is $(I,d)$-disjunct, the algorithm that finds the defective items simply starts with $X=S$ and then for every negative answer, $T(I,a)=0$, removes from $X$ all the items $i$ where $a_i=1$. From the property of $(I,d)$-disjunct all the non-defective items are removed, and $X$ will eventually contain the defective items. This can be done in linear time in the size of $M$. This is why the property of {\it disjunctness} is well studied in the literature~\cite{BBKT95,BKB95,DH06,H00,HL01,HL03,HL04}. It is well known (and very easy to see) that if $M$ is $(I,d)$-disjunct then $M$ is $(I,d)$-separable, and therefore, this algorithm is required to do more tests to get this property (with probability at least $1-\delta$). For completeness, we show in Section~\ref{S1} that the best probability for this property is $p=1-1/(d+1)$ and the number of tests $m$ satisfies $\lim_{n\to\infty} {m}/{\ln n}=\gamma_{d}=ed+{(e-1)}/{2}+O({1}/{d})$ (see the definition of $\gamma_d$ in (\ref{gamma})) which is greater than the minimum number of test in (\ref{fifi}) by the constant $e$.

In this paper we give a new algorithm that runs in linear time. Our test matrix $M$ has the following property with probability at least $1-\delta$:  It is $(I,d)$-separable test matrix and $(I,d,i)$-disjunction for at least $n-n^{1/d}$ non-defective items. Therefore one can eliminate all the non-defective items except at most $n^{1/d}$ items. Then, since $M$ is $(I,d)$-separable, the exhaustive search of the defective items in $n^{1/d}$ items takes linear time. We call a test matrix with such a property, a {\it semi-disjunct test matrix}.
We show that the number of tests $m$ in this algorithm satisfies $\lim_{n\to\infty} {m}/{\ln n}=\gamma_{d-1}+O(1/d)$ which is asymptotically the same as the one for the $(I,d)$-separable test matrix. Our analysis shows that the number of tests in our algorithm is better than the one with the disjunction property even for small $d$.

There are also other types of randomized algorithms in the literature, and they all study the property
of disjunctness. For example, the {\it Random $r$-size design} (RrSD) are algorithms where each row in $M$ is a random vector $a\in\{0,1\}^n$ of weight (number of ones in $a$) $r$. Due to the messiness of the analysis of those algorithms, only limited results have been obtained. See for example~\cite{BBKT95,BKB95,H00,HL01,HL03,HL04}. In this paper we show that, for large $n$, the optimal number of tests in the RrSD algorithms converges to the optimal number of tests in the RID model.

One advantage of the RID and RrSD algorithms over the other randomized algorithm is that, in parallel machines, the tests can be generated by different processors (or laboratories) without any communication between them. In this model all the machines uses the same distribution, draw a sample and do the test. Those algorithms are {\it strongly non-adaptive} in the sense that the rows of the matrix $M$ can be also non-adaptively generated in one parallel step. Therefore, our algorithm can also find the defective items in parallel in time $O(d\log d\log n)$ and one round of parallel queries.

There are other relaxations of the separability in the literature that give other algorithms in different models that do $O(d\log n)$ tests. Aldridge et al.,~\cite{ABG14}, and De Bonis and Vaccaro~\cite{BV15} define $\epsilon$-almost $d$-separable matrices. This gives a deterministic (non-polynomial time) algorithm that does $O(d\log n)$ tests that for $1-\epsilon$ fraction of the sets of defective items $I$ succeeds to detect $I$. In \cite{BM15}, Barg and Mazumdar defines the $(d,\epsilon)$-disjunct matrices that identify a uniform random set $I$ of defective items with false positive probability of an item at most $\epsilon$.

Our paper is organized as follows. In Section~2 we give the analysis for the $(I,d)$-disjunct test matrix and give the folklore algorithm. In Section 3 we give the analysis for the $(I,d)$-separable test matrix. In Section~4 we give our new algorithm. Then in Section~5 we show the conversion of RrSD to RID. Section~6 contains some open problems.

\section{$(I,d)$-Disjunct Matrix}\label{S1}
In this section, we find the probability $p$ for which a test matrix $M$ with a minimum number of rows (tests) is $(I,d)$-disjunct with probability at least $1-\delta$. The results in this section can be found in~\cite{BBKT95,DH06,HL04}.
We give it here for completeness.

We may assume w.l.o.g that the set of defective items is $I=\{1,2,\ldots,d'\}$ for some $d'\le d$.
Recall that $M$ is $(I,d)$-disjunct if for every non-defective item $i\in [d'+1,n]$ there is a row $j$ in $M$ such that $M_{j,1}=\cdots =M_{j,d'}=0$ and $M_i=1$.
For a random $m\times n$ test matrix $M$ with entries that are chosen independently where each entry is $0$ with probability $p$ and $1$ with probability $1-p$ we have
\begin{eqnarray*}
\Pr[M&&\!\!\!\!\!\!\!\! \mbox{\ is not $(I,d)$-disjunct}]\\
&=&\Pr[(\exists i\in [d'+1,n])(\forall j)\mbox{\ not\ } (M_{j,1}=\cdots=M_{j,d'}=0, M_{j,i}=1)]\\
&\le & n (1-p^{d'}(1-p))^m\le n (1-p^{d}(1-p))^m.
\end{eqnarray*}
To have a success probability at least $1-\delta$, we need that $n (1-p^d(1-p))^m\le \delta$ and therefore (here and throughout the paper $\ln(x)^{-1}=\ln(1/x)$ and not $1/\ln(x)$)
$$m\ge \frac{\ln n+\ln (1/\delta)}{\ln \left(1-p^d(1-p)\right)^{-1}}.$$
To minimize $m$ we need to minimize $1-p^d(1-p)$. From the first derivative, we get that the optimal $m$ is obtained when $p=d/(d+1)$. See also \cite{BBKT95}, Theorem~3.6 in \cite{HL04} and Theorem 5.3.6 in~\cite{DH06}. Therefore

\begin{theorem} \cite{BBKT95} Let $p=1-1/(d+1)$ and
$$m=\gamma_{d}\cdot (\ln n+\ln(1/\delta)) $$
where
\begin{eqnarray}\label{gamma}
\gamma_d=\frac{1}{\ln \left(1-\frac{1}{d}\left(1-\frac{1}{d+1}\right)^{d+1}\right)^{-1}}.
\end{eqnarray}
With probability at least $1-\delta$ the $m\times n$ test matrix $M$ is $(I,d)$-disjunct.
\end{theorem}
In particular,

\begin{theorem}
There is a linear time, $O(dn\log n)$, randomized algorithm that does
$$m= \gamma_d\cdot ({\ln n+\ln (1/\delta)})$$
tests
and with probability at least $1-\delta$ finds the defective items.
\end{theorem}

The algorithm is in Figure~\ref{Alg0}.
\ignore{The following table gives the coefficient of $\ln n +\ln(1/\delta)$ in the optimal $m$ for small $d$

\begin{center}
\begin{tabular}{|c|c|}
\ $d$\ & $\gamma_d$ \\
\hline
 2 & $6.2366$\\
 \hline
 3 & $8.9722$\\
 \hline
 4 & $11.6999$\\
 \hline
 5 & $14.4241$\\
 \hline
 6 & $17.1465$\\
 \hline
 7 & $19.8678$\\
 \hline
 \end{tabular}
\end{center}
}

To determine  the asymptotic behavior of $m$ we use (\ref{EQ5}) and (\ref{EQ2}) in Lemma~\ref{est} in the Appendix and get
\begin{eqnarray}\label{est1}
\lim_{n\to \infty}\frac{m}{\ln n}= \gamma_d=ed+\frac{e-1}{2}-\frac{e^2+2}{24\cdot e}\cdot \frac{1}{d}+O\left(\frac{1}{d^2}\right).
\end{eqnarray}

\begin{figure}[h!]
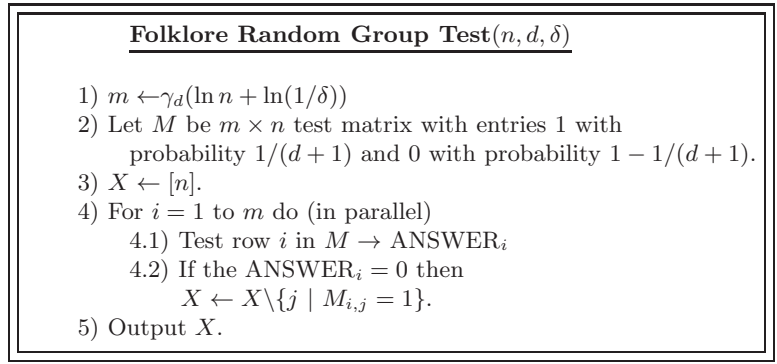

  \begin{center}
   \fbox{\fbox{\begin{minipage}{28em}
  \begin{tabbing}
  xxxx\=xxxx\=xxxx\=xxxx\= \kill
  \>\> \underline{{\bf Folklore Random Group Test$(n,d,\delta)$}}\\ \\
  \> 1) $m\gets${$ \gamma_d(\ln n+\ln(1/\delta))$}  \\
  \> 2) Let $M$ be $m\times n$ test matrix with entries $1$ with\\
  \> \> probability $1/(d+1)$ and $0$ with probability $1-1/(d+1)$.\\
  \> 3) $X\gets [n]$.\\
  \> 4) For $i=1$ to $m$ do (in parallel)\\
  \>\> 4.1) Test row $i$ in $M\to$ ANSWER$_i$\\
  \>\> 4.2) If the ANSWER$_i=0$ then\\
  \>\>\> $X\gets X\backslash \{j\ |\ M_{i,j}=1\}$.\\
  \> 5) Output $X$.
  \end{tabbing}\end{minipage}}}
  \end{center}
	\caption{An algorithm using disjunct test matrix.}
	\label{Alg0}
	\end{figure}

\section{$(I,d)$-Separable Matrix}
In this section, we find the probability $p$ for which a test matrix $M$ with a minimum number of rows (tests) is $(I,d)$-separable with probability at least $1-\delta$.
We show that, using the union bound method, the optimal number of tests is obtained when $p=1-1/d$ and the number of tests is (see (\ref{gamma}))
$$m=\gamma_{d-1}({\ln n+\ln(1/\delta)+O(d)}).$$ Notice that separability is the minimal condition we need to make sure that answers to the tests uniquely determine the defective items. Unfortunately, there is no polynomial time algorithm to find the defective items from separable matrices.
In the next section, we show that the same number of tests (up to additive $o(d)$ in $\gamma_{d-1}$) can be achieved with what we will call the semi-disjunct matrices and with such matrix one can find the defective items in linear time. So the goal of this section is to give us the minimum number of tests that is needed, and the next section will show that it is asymptotically achievable with linear time decoding.

Let $I\subseteq [n]$. Recall that for an assignment $b\in \{0,1\}^n$, $T(I,b)=1$ if $b_j=1$ for some $j\in I$ and $0$ otherwise. This is the result of the test of $b$ when $I$ is the set of defective items. For an $m\times n$ test matrix $M$, let $T(I,M)$ the column vector $(T(I,M_{i}))_{i=1,\ldots,m}$ where $M_{i}$ is the $i$th row of $M$. Recall that $M$ is $(I,d)$-separable matrix if for any set $J\not=I$ of up to $d$ items $T(I,M)\not= T(J,M)$.

Suppose $I=\{1,2,\ldots,d_1\}$ are the defective items where $d_1\le d$. For $d_2\le d$ and $0\le k\le \min(d_1,d_2)$ let ${\cal B}(n,d_2,k)$ be the set of all sets $J\subseteq [n]$ of size $|J|=d_2$ where $|I\cap J|=k$. The number of such sets is
$$|{\cal B}(n,d_2,k)|={d_1\choose k}{n-d_1\choose d_2-k}.$$

Now let $a\in \{0,1\}^n$ be a random assignment where $a_i=0$ with probability $p$ and $1$ with probability $1-p$.
The probability that $T(I,a)\not= T(J,a)$ is
the probability that $a_i=0$ for all $i\in I\cap J$ and either (1) $a_i=0$ for all $i\in I\backslash J$ and $a_j=1$ for some $j\in J\backslash I$ or (2) $a_i=0$ for all $i\in J\backslash I$ and $a_j=1$ for some $j\in I\backslash J$. Therefore
\begin{eqnarray}
\Pr[T(I,a)= T(J,a)]&=&1-\left((1-p^{d_1-k})p^{d_2}+(1-p^{d_2-k})p^{d_1}\right)\nonumber\\
&=& 1-p^{d_2}-p^{d_1}+2p^{d_1+d_2-k}.\label{ddd}
\end{eqnarray}
Therefore for a random $m\times n$ test matrix $M$, the probability that $M$ is not $(I,d)$-separable is the probability that $T(I,M)= T(J,M)$ for some $J\not=I$. This is
\begin{eqnarray*}
\Pr[&\!\! (\exists d_2,k)&(\exists J\in {\cal B}(n,d_2,k))\ \  T(I,M)= T(J,M)]\\
&\le&\sum_{d_2,k}{d_1\choose k}{n-d_1\choose d_2-k} \left(1-p^{d_2}-p^{d_1}+2p^{d_1+d_2-k}\right)^m.\\
&\le& d^22^{d}\max_{d_1,d_2,k} n^{d_2-k} \left(1-p^{d_2}-p^{d_1}+2p^{d_1+d_2-k}\right)^m\\
&\le& d^22^{d}\max_{d_1,d_2,k} P(n,d_1,d_2,k,p)
\end{eqnarray*}
where
$$P(n,d_1,d_2,k,p):=n^{d_2-k} \left(1-p^{d_2}-p^{d_1}+2p^{d_1+d_2-k}\right)^m.$$
Denote
$$\Pi= \max_{d_1,d_2,k} P(n,d_1,d_2,k,p).$$
The above implies
\begin{lemma}\label{lem1} If $\Pi\le \delta/(d^22^d)$ then with probability at least $1-\delta$ the test matrix $M$ is $(I,d)$-separable.
\end{lemma}
To find $\Pi$, we prove in the Appendix that
\begin{lemma}\label{lem} We have
\begin{eqnarray*}
\Pi&=&\max_{d_1,d_2,k} P(n,d_1,d_2,k,p)\\
&=&\max_{0\le w\le d-1}\max\{(P(n,d,d,w,p), P(n,w,d,w,p)\}\\
&=& \max_{0\le w\le d-1}\max\{n^{d-w}(1-2p^d+2p^{2d-w})^m, n^{d-w}(1+p^d-p^w)^m\}
\end{eqnarray*}
\end{lemma}

Now by Lemma~\ref{lem1} and \ref{lem} it immediately follows
\begin{lemma} If $$m=\max_{0\le w\le d-1}\frac{(d-w)\ln n+\ln(1/\delta)+\ln(d^22^d)}{\min(\ln(1-2p^d+2p^{2d-w})^{-1},\ln(1+p^d-p^w)^{-1})}$$ then the probability that random $m\times n$ test matrix $M$ is $(I,d)$-separable is at least~$1-\delta$.
\end{lemma}

In particular, for the minimum number of tests $m$ we have
$$\lim_{n\to\infty}\frac{m}{\ln n}=\min_p \max_{0\le w\le d-1}\max\left(T_1(d,w,p),T_2(d,w,p)
\right)$$ where
$$T_1(d,w,p)=\frac{d-w}{\ln(1+p^d-p^w)^{-1}}, \ \ \ T_2(d,w,p)=\frac{d-w}{\ln(1-2p^d+2p^{2d-w})^{-1}}.$$

Therefore, our goal now is to find the probability $p$ that minimizes
$$T=\max_{0\le w\le d-1}\max\left(T_1(d,w,p),T_2(d,w,p)
\right).$$
The probability that gives the minimum $T$ is denoted by $p^*$.
It is easy to see that
$$T^*_1(d,w):=\min_{0\le p\le 1} T_1(d,w,p)=\frac{d-w}{\ln(1+p_{1,w}^d-p_{1,w}^w)^{-1}}$$ with the global minimum point $$p_{1,w}=\left(\frac{w}{d}\right)^{\frac{1}{d-w}}$$
and
$$T^*_2(d,w):=\min_{0\le p\le 1} T_2(d,w,p)=\frac{d-w}{\ln(1-2p_{2,w}^d+2p_{2,w}^{2d-w})^{-1}}$$
with the global minimum point  $$p_{2,w}=\left(\frac{d}{2d-w}\right)^{\frac{1}{d-w}}.$$

To find $p^*$ we will first use the following simple fact
\begin{lemma}\label{help} Let $f_0,f_1,\ldots,f_t$ be functions $[0,1]\to \Re\cup \{\infty\}$. If $x_0$ is a global minimum point for $f_0$ and $f_0(x_0)> f_i(x_0)$ for all $i=1,\ldots,t$ then the global minimum point of $f=\max_{0\le i\le t}f_i$ is $x_0$ and $\min_{0\le x\le 1} f=f_0(x_0)$.
\end{lemma}
\begin{proof} From the definition of $f$ we have $f(x)\ge f_0(x)\ge f_0(x_0)$ for all $x$. On the other hand since $f_0(x_0)> f_i(x_0)$ for all $i=1,\ldots,t$ we have $f(x_0)=\max_{0\le i\le t}f_i(x_0)=f(x_0)$. Therefore $x_0$ is global minimum of $f$.\qed
\end{proof}

In particular, we now show that
\begin{lemma} We have $p^*=p_{1,d-1}=1-1/d$ and
\begin{eqnarray*}
\lim_{n\to\infty}\frac{m}{\ln n}&=&\min_p T=T_1^*(d,d-1)=T_1(d,d-1,p^*)\\
&=& \frac{1}{\ln\left(1-\frac{1}{d}\left(1-\frac{1}{d}\right)^{d-1}\right)^{-1}}
=\frac{1}{\ln\left(1-\frac{1}{d-1}\left(1-\frac{1}{d}\right)^{d}\right)^{-1}}=\gamma_{d-1}.
\end{eqnarray*}
\end{lemma}
\begin{proof} We use Lemma~\ref{help}. The point $p^*=p_{1,d-1}=1-1/d$ is a global minimum point of $T_1(d,d-1,p)$. We now show that $T_1(d,w,p^*)\le T_1(d,d-1,p^*)$ and $T_2(d,w,p^*)\le T_1(d,d-1,p^*)$ for all $w=0,1,\ldots,d-1$.

The proof of the fact $T_1(d,w,p^*)\le T_1(d,d-1,p^*)$ is in Lemma~\ref{Ap01} in the Appendix
and the proof of the fact $T_2(d,w,p^*)\le T_1(d,d-1,p^*)$ is in Lemma~\ref{App02} in the Appendix.\qed
\end{proof}

This implies
\begin{theorem}\label{Th3}
Let $p=1-1/d$ and
$$m=\gamma_{d-1}({\ln n+\ln(1/\delta)+O(d)}) $$
where $$\gamma_{d-1}=\frac{1}{\ln\left(1-\frac{1}{d-1}\left(1-\frac{1}{d}\right)^{d}\right)^{-1}}.$$ With probability at least $1-\delta$ the $m\times n$ test matrix $M$ is $(I,d)$-separable.
\end{theorem}
By (\ref{est1}) we get
\begin{eqnarray}\label{est2}
\gamma_{d-1}=ed-\frac{e+1}{2}+O\left(\frac{1}{d}\right).
\end{eqnarray}

The following table gives the coefficient of $\ln n +\ln(1/\delta)$

\begin{center}
\begin{tabular}{|c|c|c|}\hline
\ $d$\ & $(I,d)$-Disjunct & $(I,d)$-Seperable\\
\hline\hline
 2 & $6.2366$& $3.4761$\\
 \hline
 3 & $8.9722$& $6.2366$\\
 \hline
 4 & $11.6999$& $8.9722$\\
 \hline
 5 & $14.4241$& $11.6999$\\
 \hline
 6 & $17.1465$ & $14.4241$\\
 \hline
 7 & $19.8678$& $17.1465$\\
 \hline
 \end{tabular}
\end{center}

The difference of the coefficients in the sizes is
$$\gamma_{d}-\gamma_{d-1}= e+O\left(\frac{1}{d}\right)\longrightarrow e.$$

\section{An Algorithm with Linear Decoding Time}
In this section we give our algorithm, show that it does a minimum number of tests and runs in linear time.

We say that a matrix $M$ is {\it $(I,d)$-semi-disjunct matrix} if it is $(I,d)$-separable and $(I,d,i)$-disjunct for at least $n-n^{1/d}$ items $i$.

We prove
\begin{theorem}\label{TH01}
Let $p=1-1/d$ and
$$m=\frac{\left(1-\frac{1}{d}\right)\ln n+\ln \frac{1}{\delta}+O(d)}{\ln \left(1-\frac{1}{d}\left( 1-\frac{1}{d}\right)^d\right)^{-1}}.$$
With probability at least $1-\delta$ the $m\times n$ test matrix $M$ is $(I,d)$-semi-disjunct matrix.

\end{theorem}

In particular we prove
\begin{theorem}\label{TH02}
There is a linear time non-adaptive randomized algorithm that does
$$m= \frac{\left(1-\frac{1}{d}\right)\ln n+\ln \frac{1}{\delta}+O(d)}{\ln \left(1-\frac{1}{d}\left( 1-\frac{1}{d}\right)^d\right)^{-1}}$$
tests
and with probability at least $1-\delta$ finds the defective items.

In particular,
$$\frac{m}{\ln n}=\frac{1-\frac{1}{d}}{\ln \left(1-\frac{1}{d}\left( 1-\frac{1}{d}\right)^d\right)^{-1}}=
\gamma_{d-1}+O\left(\frac{1}{d}\right)\to \gamma_{d-1}$$ is asymptotically equal to the
constant that is achieved by the $(I,d)$-separability property.
\end{theorem}

We first prove Theorem~\ref{TH01}

\noindent {\bf Proof of Theorem~\ref{TH01}}
\begin{proof}
In Lemma~\ref{Las} in the Appendix, we show that
$$\frac{\left(1-\frac{1}{d}\right)}{\ln \left(1-\frac{1}{d}\left( 1-\frac{1}{d}\right)^d\right)^{-1}}\ge \gamma_{d-1}.$$
Therefore $m\ge \gamma_{d-1}(\ln n+\ln (1/\delta)+O(d))$ and by Theorem~\ref{Th3}, with probability at least $1-(\delta/2)$, $M$ is $(I,d)$-separable. Note that $\delta/2$ adds a constant to the second term that is swallowed by $O(d)$.

Let w.l.o.g $I=\{1,\ldots,d'\}$ where $d'\le d$. Define for every $k\not \in I$, a random variable $X_k$ that takes the values $\{0,1\}$ and is equal to $0$ if there is a row $j$ in the test matrix such that $M_{j,i}=0$ for all $i\in I$ and $M_{j,k}=1$. That is, $M$ is $(I,d,k)$-disjunct. If so then we know that $k\not \in I$.
Let $D=\{k|X_k=1\}$. The set $D$ is the set of items that are not $(I,d,k)$-disjunct. We now find the expected size of $D$. We first have
$$\E[X_k]\le (1-p^d (1-p))^m\le \frac{\delta}{2n^{1-\frac{1}{d}}}.$$
Then
$$\E[|D|]=\E\left[\sum_{i\not\in I} X_k\right]=\sum_{i\not\in I} \E [X_k]\le (\delta/2) n^{1/d}.$$
By Markov bound
$$\Pr[|D|\ge n^{1/d}]\le \frac{\delta}{2}$$ and therefore with probability at least $1-(\delta/2)$ we have that $|D|\le n^{1/d}$.

Therefore, with probability at least $1-\delta$ the test matrix $M$ is $(I,d)$-separable and $(I,d,i)$-disjunct for at least $|[n]\backslash D|\ge n-n^{1/d}$ items.\qed
\end{proof}

Theorem~\ref{TH02} is proved in the next section.

In the table we compare the size of the $(I,d)$-separable, $(I,d)$-disjunct and $(I,d)$-semi-disjunct.

\begin{center}
\begin{tabular}{|c|c|c|c|c|}\hline
\ $d$\ & $(I,d)$-Disjunct & $(I,d)$-Separable & $(I,d)$-Semi-Disjunct\\
\hline\hline
 2 & $6.2366$& $3.4761$ & $3.7444$ \\
 \hline
 3 & $8.9722$& $6.2366$ & $6.4109$ \\
 \hline
 4 & $11.6999$& $8.9722$ & $9.1013$ \\
 \hline
 5 & $14.4241$& $11.6999$ & $11.8025$ \\
 \hline
 6 & $17.1465$ & $14.4241$ & $14.5093$ \\
 \hline
 7 & $19.8678$& $17.1465$ & $17.2193$ \\
 \hline
 \end{tabular}
\end{center}

For the $(n,d)$-semi-disjunct matrix, the constant in the table is
\begin{eqnarray*}
\frac{m}{\ln n}&=& \frac{\left(1-\frac{1}{d}\right)}{\ln \left(1-\frac{1}{d}\left( 1-\frac{1}{d}\right)^d\right)^{-1}}\\
&=&
ed-\frac{e+1}{2}-\frac{e^2-12e+2}{24e}\frac{1}{d}+O\left(\frac{1}{d^2}\right)= \gamma_{d-1}+O\left(\frac{1}{d}\right)
\end{eqnarray*}
and therefore the constant for the $(I,d)$-Separable (the third column in the table) converges to the constant of the $(I,d)$-Semi-Disjunct (the forth column in the table).

\subsection{The Algorithm}
In our algorithm, we choose $p=1-1/d$ and
$$m=\frac{\left(1-\frac{1}{d}\right)\ln n+\ln \frac{1}{\delta}+O(d)}{\ln \left(1-\frac{1}{d}\left( 1-\frac{1}{d}\right)^d\right)^{-1}}.$$

\setcounter{theorem}{4}

We now prove
\begin{theorem}
There is a linear time, $O(dn\log n)$, randomized algorithm that does (Here $o(1)=O(1/d)$)
$$m= \frac{\left(1-\frac{1}{d}\right)\ln n+\ln \frac{1}{\delta}+O(d)}{\ln \left(1-\frac{1}{d}\left( 1-\frac{1}{d}\right)^d\right)^{-1}}= (\gamma_{d-1}+o(1))\cdot {\ln n+\gamma_{d}\ln (1/\delta)}+O(\gamma_dd)$$
tests
and with probability at least $1-\delta$ finds the defective items.
\end{theorem}
\begin{proof}
Consider the random $m\times n$ test matrix $M$ with entries $0$ with probability $p=1-1/d$ and $1$ with probability $1/d$. Let $I$ be set of defective items. If $|I|< d$ then since $m\ge \gamma_{d-1} (\ln n+\ln(1/\delta)+O(d))$, by Theorem~\ref{Th3}, with probability at least $1-\delta$ the test matrix $M$ is $(I,d-1)$-disjunct test matrix. Therefore, if the number of defective items is less than $d$ then with probability at least $1-\delta$ the algorithm in Section \ref{S1} finds the defective items. Therefore we may assume that the number of defective items is exactly $d$.

Now, if the number of defective items is $d$ then, by Theorem~\ref{TH01}, with probability at least $1-\delta$, the test matrix $M$ is $(I,d)$-semi-disjunct. Therefore by the algorithm in Section~\ref{S1} all the non-defective items are eliminated except at most $n^{1/d}$ items. Since $M$ is also $(n,I)$-separable matrix the only set $J$ that satisfies $T(J,M)=T(I,M)$ is $I$. So we can exhaustively search for the $d$ defective items in $n^{1/d}$ items. This takes time
$${n^{1/d}\choose d}dm\le \left(\frac{en^{1/d}}{d}\right)^dd^2\log n=e\left(\frac{e}{d}\right)^{d-1}dn\log n=O(dn\log n)$$ and therefore linear time.\qed
\end{proof}

The algorithm is in Figure~\ref{Alg}. In the algorithm, $M^{(i)}$ is the $i$th column of $M$. ANSWER$_i$ is $0$ if the test is negative and $1$ if it is positive. The correctness of the algorithm follows from the above analysis.

\begin{figure}[h!]
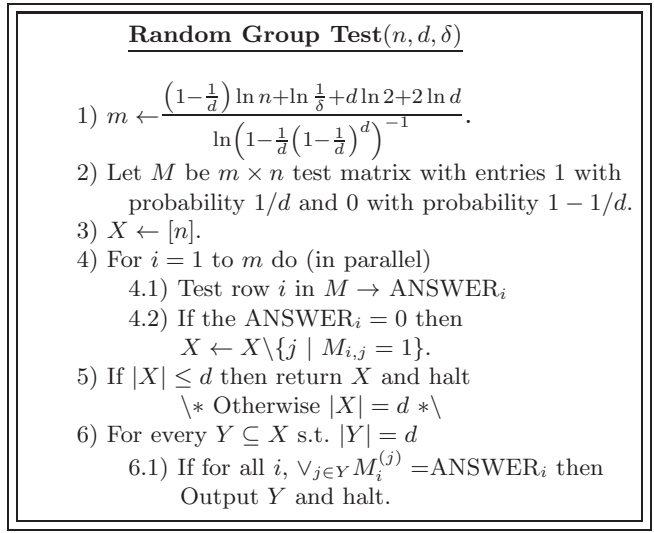

  \begin{center}
   \fbox{\fbox{\begin{minipage}{28em}
  \begin{tabbing}
  xxxx\=xxxx\=xxxx\=xxxx\= \kill
  \>\> \underline{{\bf Random Group Test$(n,d,\delta)$}}\\ \\
  \> 1) $m\gets${\large $ \frac{\left(1-\frac{1}{d}\right)\ln n+\ln \frac{1}{\delta}+d\ln 2+2\ln d}{\ln \left(1-\frac{1}{d}\left( 1-\frac{1}{d}\right)^d\right)^{-1}}.$}  \\
  \> 2) Let $M$ be $m\times n$ test matrix with entries $1$ with\\
  \> \> probability $1/d$ and $0$ with probability $1-1/d$.\\
  \> 3) $X\gets [n]$.\\
  \> 4) For $i=1$ to $m$ do (in parallel)\\
  \>\> 4.1) Test row $i$ in $M\to$ ANSWER$_i$\\
  \>\> 4.2) If the ANSWER$_i=0$ then\\
  \>\>\> $X\gets X\backslash \{j\ |\ M_{i,j}=1\}$.\\
  \> 5) If $|X|\le d$ then return $X$ and halt\\
  \> \>\> $\backslash*$ Otherwise $|X|=d$ $*\backslash$ \\
  \> 6) For every $Y\subseteq X$ s.t. $|Y|=d$\\
  \>\>  6.1) If for all $i$, $\vee_{j\in Y} M^{(j)}_i=$ANSWER$_i$ then\\
  \>\> \> Output $Y$ and halt.
  \end{tabbing}\end{minipage}}}
  \end{center}
	\caption{An algorithm using semi-disjunct test matrix.}
	\label{Alg}
	\end{figure}

\section{Open Problems}
In this section we give some open problems

\begin{enumerate}
\item Our analysis for $(I,d)$-separable matrices gives the bound
\begin{eqnarray}\lim_{n\to\infty} \frac{m_{\RID}}{\ln n}\le \gamma_{d-1} :=ed-\frac{e+1}{2}+O\left(\frac{1}{d}\right)\end{eqnarray}
We still do not know whether this bound is tight in the RID model.
Also, no lower bound is known except the information theoretic lower bound $d\log n$.

\item We show in Section~\ref{S1} that the best probability for the disjunctness property is $p=1-1/(d+1)$. This is based on maximizing the expected number of elements that are eliminated for each test. This is not necessarily the best probability for obtaining a minimum number of tests. Find the probability that gives the minimum number of tests. Find a lower bound for the number of tests.

\item In Section 5 we show that the RrSD model converges to the RID model. Consider the most general strong non-adaptive model (RSN model) where each row in $M$ is chosen randomly independently according to some distribution $D$ over $\{0,1\}^n$. Show that the algorithm that do the minimum number of tests is the RsSD algorithm with $r=n/d$.
\end{enumerate}

\newpage
\section*{Appendix}
\setcounter{equation}{0}

{\bf Proof of Lemma~\ref{lem}}

We need to prove
\begin{eqnarray*}
\Pi&=&\max_{d_1,d_2,k} P(n,d_1,d_2,k,p)\\
&&\ \ \ \ \ \ =\max_{0\le w\le d-1}\max\{(P(n,d,d,w,p), P(n,w,d,w,p)\}.
\end{eqnarray*}
\begin{proof} If $1-2p^{d_2-k}\ge 0$ then $d_2-k\ge 1$ and
\begin{eqnarray*}
P(n,d_1,d_2,k,p)&=& n^{d_2-k} \left(1-p^{d_2}-p^{d_1}+2p^{d_1+d_2-k}\right)^m\\
&=& n^{d_2-k} \left(1-p^{d_2}-p^{d_1}(1-2p^{d_2-k})\right)^m\\
&\le & n^{d_2-k} \left(1-p^{d_2}-p^{d}(1-2p^{d_2-k})\right)^m\\
&\le & n^{d_2-k} \left(1-p^{d}-p^{d}(1-2p^{d_2-k})\right)^m\\
&= & n^{d_2-k} \left(1-2p^{d}+2p^{d+d_2-k})\right)^m\\
&=& P(n,d,d,d-(d_2-k),p)\\
&\le& \max_{0\le w\le d-1}P(n,d,d,w,p).
\end{eqnarray*}
If $1-2p^{d_2-k}< 0$ and $d_2>k$ then
\begin{eqnarray*}
P(n,d_1,d_2,k,p)&=& n^{d_2-k} \left(1-p^{d_2}-p^{d_1}+2p^{d_1+d_2-k}\right)^m\\
&=& n^{d_2-k} \left(1-p^{d_2}-p^{d_1}(1-2p^{d_2-k})\right)^m\\
&\le & n^{d_2-k} \left(1-p^{d_2}-p^{k}(1-2p^{d_2-k})\right)^m\\
&= & n^{d_2-k} \left(1+p^{d_2}-p^{k}\right)^m\\
&= & n^{d_2-k} \left(1+p^k(p^{d_2-k}-1)\right)^m\\
&\le & n^{d_2-k} \left(1+p^{k+d-d_2}(p^{d_2-k}-1)\right)^m\\
&= & n^{d_2-k} \left(1+p^{d}-p^{d-(d_2-k)}\right)^m\\
&=& P(n,d-(d_2-k),d,d-(d_2-k),p)\\
&\le& \max_{0\le w\le d-1}P(n,w,d,w,p).
\end{eqnarray*}
If $d_2=k$ then $d_1>d_2$ (otherwise $J=I$) and
\begin{eqnarray*}
P(n,d_1,d_2,k,p)&=&\left(1-p^{d_2}+p^{d_1}\right)^m\\
&\le& (1-p^{d_1-1}+p^{d_1})^m\\
&=& (1-p^{d_1-1}(1-p))^m\\
&\le& (1-p^{d-1}(1-p))^m\\
&=&P(n,d,d-1,d-1,p)\\
&\le& P(n,d-1,d,d-1,p)\\
&\le& \max_{0\le w\le d-1}P(n,w,d,w,p).\qed
\end{eqnarray*}
\end{proof}

The following are used to estimate some of the expressions in the paper

\begin{lemma}\label{est}
\begin{eqnarray}\label{EQ1} \left(1-\frac{1}{d}\right)^d=\frac{1}{e}-\frac{1}{2e}
\frac{1}{d}-\frac{5}{24e}\frac{1}{d^2}+O\left(\frac{1}{d^3}\right)
\end{eqnarray}
\begin{eqnarray}\label{EQ2}\left(1-\frac{1}{d+1}\right)^{d+1}=\frac{1}{e}-\frac{1}{2e}
\frac{1}{d}+\frac{7}{24e}\frac{1}{d^2}+O\left(\frac{1}{d^3}\right)\mbox{\ and}\nonumber\\
 \left(1-\frac{1}{d+1}\right)^{-(d+1)}=e+\frac{e}{2}
\frac{1}{d}+\frac{e}{24}\frac{1}{d^2}+O\left(\frac{1}{d^3}\right).
\end{eqnarray}
\begin{eqnarray}\label{EQ3}
\frac{1}{d+1}=\frac{1}{d}-\frac{1}{d^2}+\frac{1}{d^3}-\cdots \ \ \mbox{and}\ \  \frac{1}{d-1}=\frac{1}{d}+\frac{1}{d^2}+\frac{1}{d^3}+\cdots
\end{eqnarray}
\begin{eqnarray}\label{EQ4}\ln(1-x)^{-1}=x+\frac{1}{2}x^2+\frac{1}{3}x^3+O(x^4).\end{eqnarray}
\begin{eqnarray}\label{EQ5}\frac{1}{\ln(1-x)^{-1}}=\frac{1}{x}-\frac{1}{2}-\frac{1}{12}x-\frac{1}{24}x^2-\frac{19}{720}x^3+O(x^4).\end{eqnarray}
\begin{eqnarray}\label{EQ6}\frac{1}{A+Bx+Cx^2+O(x^3)}=\frac{1}{A}-\frac{B}{A^2}x+\frac{B^2-CA}{A^3}x^2+O(x^3).\end{eqnarray}
\begin{eqnarray}\label{EQ7}
e^{A+Bx+Cx^2+O(x^3)}=e^A\left(1+Bx+\left(\frac{B^2}{2}+C\right)x^2+O(x^3)\right).
\end{eqnarray}
\end{lemma}

The following lemma will be frequently used in the sequel
\begin{lemma}\label{main}
Let $0<z_1<z_2<1$ then
$$\frac{\ln(1-z_1)^{-1}}{\ln(1-z_2)^{-1}}\le \frac{z_1}{z_2}$$
and
$$\frac{\ln(1-z_2)^{-1}}{\ln(1-z_1)^{-1}}\le \frac{z_2}{z_1}.$$
\end{lemma}
\begin{proof}
By (\ref{EQ4}) in Lemma~\ref{est} we have
\begin{eqnarray*}
\frac{\ln(1-z_1)^{-1}}{\ln(1-z_2)^{-1}}=\frac{z_1}{z_2}\frac{\sum_{i=1}^\infty \frac{1}{i} z_1^{i-1}}{\sum_{i=1}^\infty \frac{1}{i} z_2^{i-1}}\le \frac{z_1}{z_2}.\qed
\end{eqnarray*}
\end{proof}

\begin{lemma}\label{Ap01}
  We have $T_1(d,w,p^*)\le T_1(d,d-1,p^*).$
\end{lemma}
\begin{proof}
By Lemma~\ref{main}, we have
  \begin{eqnarray*}
  (d-w)\frac{T_1(d,d-1,p^*)}{T_1(d,w,p^*)}=\frac{\ln\left( 1+\left(\frac{d-1}{d}\right)^d-\left(\frac{d-1}{d}\right)^w\right)^{-1}}
  {\ln\left( 1+\left(\frac{d-1}{d}\right)^d-\left(\frac{d-1}{d}\right)^{d-1}\right)^{-1}}
  \end{eqnarray*}
  \begin{eqnarray*}
  &\ge& \frac{\left(\frac{d-1}{d}\right)^w-\left(\frac{d-1}{d}\right)^d} {\left(\frac{d-1}{d}\right)^{d-1}-\left(\frac{d-1}{d}\right)^{d}}
  =(d-1)\left(\left(1+\frac{1}{d-1}\right)^{d-w}-1\right)\\
  &\ge&(d-1)\left(\left(1+\frac{d-w}{d-1}\right)-1\right)=d-w.\qed
  \end{eqnarray*}
\end{proof}

\begin{lemma}\label{App02} We have
$
T_2(d,w,p^*)\le T_1(d,d-1,p^*)
$.
\end{lemma}
\begin{proof} Since $f(x)=(1-x/z)^{1/x}$ is monotonically decreasing function for $x\in [0,z]$ we have, for $w\in [1,d-1]$,
$$\left(1-\frac{d-w}{2(d-1)}\right)^{\frac{1}{d-w}}\ge 1-\frac{1}{2(d-1)}\ge 1-\frac{1}{d}.$$
This is equivalent to
\begin{eqnarray}\label{local}
2\left(1-\left(1-\frac{1}{d}\right)^{d-w}\right)\ge \frac{d-w}{d-1}.
\end{eqnarray}
Now by (\ref{local}) and Lemma~\ref{main} we have
$$(d-w)\frac{T_1(d,d-1,p^*)}{T_2(d,w,p^*)}=\frac{\ln\left(1-2\left(\frac{d-1}{d}\right)^{d} +2\left(\frac{d-1}{d}\right)^{2d-w}\right)^{-1}}
{\ln\left(1-\frac{1}{d-1}\left(1-\frac{1}{d}\right)^d\right)^{-1}} $$
\begin{eqnarray*}
&=& \frac{\ln\left(1-2 \left(1-\left(1-\frac{1}{d}\right)^{d-w}\right)\left(1-\frac{1}{d}\right)^d\right)^{-1}}
{\ln\left(1-\frac{1}{d-1}\left(1-\frac{1}{d}\right)^d\right)^{-1}}\\
&\ge& \frac{\ln\left(1-\frac{d-w}{d-1}\left(1-\frac{1}{d}\right)^d\right)^{-1}}
{\ln\left(1-\frac{1}{d-1}\left(1-\frac{1}{d}\right)^d\right)^{-1}}\ge d-w\\
\end{eqnarray*}
\end{proof}

\begin{lemma}\label{Las}
We have
$$\frac{\left(1-\frac{1}{d}\right)}{\ln \left(1-\frac{1}{d}\left( 1-\frac{1}{d}\right)^d\right)^{-1}}\ge \gamma_{d-1}.$$
\end{lemma}
\begin{proof}
By Lemma~\ref{main} we have
\begin{eqnarray*}
\frac{1/\gamma_{d-1}}
{\ln\left(1-\frac{1}{d}\left(1-\frac{1}{d}\right)^d\right)^{-1}}&=&\frac{\ln\left(1-\frac{1}{d-1}\left(1-\frac{1}{d}\right)^d\right)^{-1}}
{\ln\left(1-\frac{1}{d}\left(1-\frac{1}{d}\right)^d\right)^{-1}}\\
&\ge & \frac{\frac{1}{d-1}}{\frac{1}{d}}=\frac{1}{1-\frac{1}{d}}.\qed
\end{eqnarray*}
\end{proof}

The following lemma can be proved by induction
\begin{lemma}\label{iiq} For $0\le x_i\le 1$ and $w_i>0$, $i=1,\ldots,n$ we have
$$\prod_{i=1}^n (x_i-w_i)\ge \left(\prod_{i=1}^n x_i\right)-\sum_{i=1}^n w_i.$$
\end{lemma}

\section{Equivalent Models}
In this section we show that, for large $n$, the number of tests in this RrSD algorithm converges to the number of tests in the RID algorithm. This shows that our RID algorithm in this paper gives the optimal number of tests in the strong non-adaptive models. I.e., in the models where the rows of the matrix $M$ is chosen randomly independently.

We now prove
\begin{theorem} In the RrSD model, for large enough $n$, the algorithm that gives the minimum number of tests is the one with $r=n/d$. The number of tests converges to the number of tests of the best RID algorithm with $p=1/d$.
\end{theorem}
\begin{proof}
Following the analysis in Section~3, let $a\in \{0,1\}^n$ be a random uniform assignment of weight $r$.
The probability that $T(I,a)\not= T(J,a)$ is
the probability that $a_i=0$ for all $i\in I\cap J$ and either (1) $a_i=0$ for all $i\in I\backslash J$ and $a_j=1$ for some $j\in J\backslash I$ or (2) $a_i=0$ for all $i\in J\backslash I$ and $a_j=1$ for some $j\in I\backslash J$. Therefore
\begin{eqnarray*}
\Pr[T(I,a)= T(J,a)]&=&1-\frac{{n-d_1-k\choose r}+{n-d_2-k\choose r}-2{n-d_1-d_2-k\choose r}}{{n\choose r}}.
\end{eqnarray*}
Now since
$$q_x:=\frac{{n-x\choose r}}{{n\choose r}}=\prod_{i=1}^x \left(1-\frac{r}{n-i+1}\right)\le \left(1-\frac{r}{n}\right)^x$$
and for large enough $n$, by Lemma~\ref{iiq},
\begin{eqnarray*}
q_x= \prod_{i=1}^x \left(1-\frac{r}{n-i+1}\right)&\ge& \prod_{i=1}^x \left(\left(1-\frac{r}{n}\right)-\frac{2(i-1)r}{n^2}\right)\\
&\ge & \left(1-\frac{r}{n}\right)^x-\sum_{i=1}^x \frac{2(i-1)r}{n^2}\\
&\ge & \left(1-\frac{r}{n}\right)^x- \frac{x^2r}{n^2}\\
\end{eqnarray*}
we have
\begin{eqnarray*}
\Pr[T(I,a)= T(J,a)]&=& 1-q_{d_1+k}-q_{d_2+k}+ 2 q_{d_1+d_2-k}\\
&=&1-p^{d_2}-p^{d_1}+2p^{d_1+d_2-k}+O\left(\frac{d^2(1-p)}{n}\right)
\end{eqnarray*}
where $p=1-r/n$.

Notice that this converges to the same probability in (\ref{ddd}) for the RID algorithm.
We have shown that the value of $p$ that gives the best number of tests is $p=1-1/d$ and therefore
$r=n/d$.\qed
\end{proof}

\ignore{We now prove
\begin{lemma}
Among the RIR algorithms, where each vector $a\in\{0,1\}^n$ is drawn randomly and independently according to a fixed distribution $D$ over $\{0,1\}^n$, the algorithm that gives the maximum amount of information about the non-defective items is the RrSD algorithm where each row in M is random uniform $a\in\{0,1\}^n$ of a weight $r=n/d$.
\end{lemma}
\begin{proof} Suppose $I=\{i_1,\ldots,i_{d-1}\}$. Any algorithm for detecting the defective items must with probability at least $1/2$ distinguish between $I$ and $I\cup \{i_{d}\}$ for all $i_d\not\in I$. The only way the algorithm can do that is if there is a row $a$ in the matrix $M$ in which $a_{i_1}=a_{i_2}=\cdots=a_{i_{d-1}}=0$ and $a_{i_d}=1$. This gives an evidence that $i_d$ is not in $I$. With one vector $a$ when $a_{i_j}=1$ for some $j$ the test is positive and the only information that is gained is that one of the items $k$ for which $a_k=1$ is defective. No information is gained about the non-defective items.

In this proof we give an algorithm of an adversary that for a given distribution $D$, gives us $i_1,\ldots,i_{d-1}$ but not $i_d$. We then show that the expected number of the non-defective items $i_d$ that are eliminated when the row is drawn according to the distribution $D$ is less than the expectation when the row is randomly and uniformly drawn from all vectors of weight $r=n/d$.
\end{proof}}

\end{document}